\documentclass[letterpaper]{article} 
\usepackage{aaai2026}  
\usepackage{times}  
\usepackage{helvet}  
\usepackage{courier}  
\usepackage[hyphens]{url}  
\usepackage{graphicx} 
\usepackage{natbib}  
\usepackage{caption} 
\usepackage{float} 
\usepackage{amsmath} 
\usepackage{mdframed} 
\usepackage{multirow} 
\usepackage{booktabs} 
\usepackage{xcolor} 
\usepackage{amsthm} 
\usepackage[linesnumbered,ruled,vlined]{algorithm2e}
\usepackage{subfig} 

\usepackage{newfloat}
\usepackage{listings}
\DeclareCaptionStyle{ruled}{labelfont=normalfont,labelsep=colon,strut=off} 
\floatstyle{ruled}
\newfloat{listing}{tb}{lst}{}
\floatname{listing}{Listing}
\newtheorem{defn}{Definition}
\newtheorem{lem}{Lemma}

\title{Optimized Algorithms for Text Clustering with LLM-Generated Constraints}
\author{
    Chaoqi Jia\textsuperscript{\rm 1},
    Weihong Wu\textsuperscript{\rm 2},
    Longkun Guo\textsuperscript{\rm 2}\thanks{Corresponding Author},
    Zhigang Lu\textsuperscript{\rm 3},
    Chao Chen\textsuperscript{\rm 1},
    Kok-Leong Ong\textsuperscript{\rm 1}
}
\affiliations{
    \textsuperscript{\rm 1} School of Accounting, Information Systems and Supply Chain, RMIT University, Melbourne, VIC 3000, Australia\\
    \textsuperscript{\rm 2} School of Mathematics and Statistics, Fuzhou University, Fuzhou 350116, China\\
    \textsuperscript{\rm 3} Western Sydney University, NSW 2751, Australia\\

    chaoqi.jia@student.rmit.edu.au, weihong.research@gmail.com, longkun.guo@gmail.com, z.lu@westernsydney.edu.au, \\
    \{chao.chen, kok-leong.ong2\}@rmit.edu.au
}

\usepackage{bibentry}

\begin{document}
\maketitle

\begin{abstract}
Clustering is a fundamental tool that has garnered significant interest across a wide range of applications including text analysis.  To improve clustering accuracy, many researchers have incorporated background knowledge, typically in the form of must‑link and cannot‑link constraints, to guide the clustering process.
With the recent advent of large language models (LLMs), there is growing interest in improving clustering quality through LLM-based automatic constraint generation. In this paper, we propose a novel constraint‑generation approach that reduces resource consumption by generating constraint sets rather than using traditional pairwise constraints. This approach improves both query efficiency and constraint accuracy compared to state‑of‑the‑art methods. We further introduce a constrained clustering algorithm tailored to the characteristics of LLM-generated constraints. Our method incorporates a confidence threshold and a penalty mechanism to address potentially inaccurate constraints. We evaluate our approach on five text datasets, considering both the cost of constraint generation and the overall clustering performance. The results show that our method achieves clustering accuracy comparable to the state-of-the-art algorithms while reducing the number of LLM queries by more than $20$ times.
\end{abstract}

\begin{links}
    \link{Code}{https://github.com/weihong-wu/LSCK-HC}
\end{links}

\section{Introduction}
\label{sec:intro}
Short text clustering (STC) is a well-known natural language processing task, widely used for analyzing brief user-generated content on platforms such as Twitter, Instagram, and Facebook~\cite{ahmed2022short}. Thanks to its simplicity and efficiency, the $k$-means algorithm is widely applied when implementing the STC tasks. To further meet clients' expectations and improve the clustering quality, people have introduced semi-supervised approaches using additional background knowledge of the raw data to assist the $k$-means clustering~\cite{cai2023review}.

For text clustering with pairwise constraints, many researchers~\cite{basu2002semi,bae2020interactive,viswanathan2024large} have incorporated background knowledge into $k$-means clustering. 
For instance, \citet{wagstaff2000clustering} introduced the concepts of cannot-link (CL) and must-link (ML) constraints and designed a clustering algorithm that enforces these constraints. Building upon this work, \citet{basu2002semi} developed algorithms that applied these constraints to textual data, such as news articles, utilizing spherical $k$-means for clustering. Following this, \citet{basu2004active} proposed an active-learning approach to generate constraints automatically, thereby reducing reliance on manual input, and integrated these constraints directly into the clustering objective. 
However, these methods still depend on constraints that are either hand-picked by experts~\cite{basu2002semi} or derived from existing labels~\cite{baumann2024algorithm}. 
In textual datasets and unlabelled domains, such as social media text, manually generating a sufficient number of constraints is prohibitively time-consuming and costly.

Recently, LLMs have been considered to serve many different short text clustering tasks, for instance, they can be used to generate pairwise constraints for clustering by In-Context Learning~\cite{zhang2023clusterllm,viswanathan2024large} and text summarization~\cite{feng2025bimark,zhang2025systematic}. However, \citet{viswanathan2024large} incurred high costs while achieving limited improvements in constrained $k$-means clustering when applying their LLM-generated constraints. To fill the gap, we argue that LLMs can offer greater benefits if the pairwise constraint format is replaced with a set-based formulation and if the clustering algorithm is designed to tolerate false in LLM-generated constraints. Based on these challenges, we summarize our main contributions as follows.

\begin{itemize}

     \item By leveraging the generation capabilities of LLMs, we provide a query selection method for decision-making to generate the CL and ML constraints.  We then apply the LLMs to measure the thresholds to categorize the must-links as hard or soft constraints, which are used in the constrained clustering process.
   
     \item Based on the LLM-generated constraints, we propose a penalty‐augmented local search clustering algorithm to handle the CL and ML constraints. Cannot-link constraints are handled via a local search with the penalty matching method, while must-link constraints are divided into two types: hard constraints guide the initial seeding represented by the mean of the set, and soft constraints correct the clustering via penalties.

    \item By evaluating the LLM‐generated clustering framework on $5$ real‐world text datasets, our results show that the clustering quality attains accuracy that is comparable to or exceeds that of the existing algorithms, while reducing LLM queries by more than $20$ times. 
\end{itemize}

\section{Related Works}
\label{sec:related}

Clustering algorithms, such as $k$-means~\cite{lloyd1982least} and $k$-means++~\cite{arthur2007kmeanspp}, are widely used in short-text clustering. Extensive research has been conducted on interactive approaches that integrate partial supervision through must-link and cannot-link constraints~\cite{bae2020interactive}. In this work, we concentrate on introducing two critical components of constrained clustering: constraint generation and clustering algorithms.

\textbf{Generated Constraints.} \citet{basu2004active} first proposed an automated strategy for generating must-link and cannot-link constraints, known as farthest-first query selection, which reduces reliance on manual labeling. Then, other typical methods have been introduced, such as the min-max approach~\cite{mallapragada2008active} and normalized point-based uncertainty~\cite{xiong2013active}. The primary objective of these methods is to select the most informative and representative examples for supervision~\cite{fu2024acdm}. In recent years, active clustering with pairwise constraints has been extensively adopted in semi-supervised clustering~\cite{yu2018semi,xiong2016active,li2019ascent,fu2024acdm}. However, these approaches still rely on expert judgment or are derived from existing labels as prior knowledge. Recent works~\cite{zhang2023clusterllm,viswanathan2024large} have explored the use of In-Context Learning (ICL) with LLMs to generate pairwise constraints, while \citet{huang2024text} employ LLMs to produce potential labels for classification tasks. In our work, we utilize the characteristics of LLMs and then design algorithms to enhance query efficiency by extending constraints from individual pairs to sets, thereby reducing the number of LLM queries required while improving the constraints' accuracy.

\textbf{Constrained Clustering Algorithms.} 
\citet{wagstaff2001constrained} incorporated must-link and cannot-link constraints into $k$-means clustering and proposed the COP-KMeans algorithm, a greedy approach designed for constraint satisfaction. Next, constrained clustering algorithms have typically categorized relationships between instances as either hard constraints which must always be satisfied~\cite{basu2002semi,ganji2016lagrangian,le2018binary,baumann2020binary,jia2023efficient}, or soft constraints which are integrated into the objective function with penalties for violations~\cite{basu2004active,davidson2005clustering,davidson2005agglomerative,pelleg2007k,baumann2022k}. More recently, \citet{baumann2024algorithm} proposed the constraints clustering algorithm utilizing integer programming to cluster and effectively manage constraints allowing for the inclusion of both hard and soft types. In addition, \citet{jia2023efficient} and \citet{guo2024efficient} proposed approaches that handle constraints as disjoint sets to prove the approximation ratio of clustering. However, existing algorithms are not specifically designed to leverage the properties of LLM-generated constraints, even though soft-constraint clustering algorithms can mitigate some misclassifications caused by the false constraint relationships. In this paper, we propose a clustering algorithm that combines LLM-generated hard and soft constraints to align with the decision-making characteristics of LLMs.

\section{Problem Formulation}
\label{sec:Pre}
In text clustering, a set of $n$ texts $T$ is provided as input, where each short text is transformed into a point in $\delta$-dimensional space $\mathcal{R}^\delta$ by embedding models, resulting in a set of representations (points) $S$. Let $\pi$ denote the mapping function from the text $t\in T$ to its embedding point $x = \pi(t) \in S$. For a  parameter $k$, the goal of the text clustering task is to assign these points to $k$ clusters, denoted by $\mathcal{A} = \{A_1, \ldots, A_k\}$, such that the objective function is minimized.

In particular, the traditional $k$-means algorithm clusters the data based on the means with the goal of minimizing the following objective function:
 \[
    \min \sum_{i=1}^{k}\sum_{x\in A_{i}\subseteq S}||x-c(A_i)||^{2}
 \]
where cluster $A_i$ represents the collection of points within $S$ and $c(A_i)$ denotes the mass center of $A_i$. Let $C$ denote the center set involving all the $c(A_i)$ for $i \in[k]$. With the same objective, the constrained $k$-means needs to additionally satisfy the constraint relationships below.

Let must-link (ML) constraints be defined as $\mathcal{X} = \{X_{1},\ldots,X_{h}\}$, where each {$X \in\mathcal{X}$} is an ML set. For each point within the ML set, the clustering considers assigning them to the same center. 
Similarly, let cannot-link (CL) constraints be defined as $\mathcal{Y} = \{Y_{1},\ldots,Y_{l}\}$, where each {$Y \in\mathcal{Y} $} represents a CL set $|Y| \leq k$, and the clustering considers assigning the points in $Y$ to different centers. Note that the hard constraints require the point assignments to strictly follow these constraints, while the soft constraints allow them to violate the constraints by paying a certain penalty.

\section{Methodology}
\label{sec:Methods}
In this section, we first propose the algorithmic framework of our approach. Then, we describe the two key stages for constrained text clustering: 1) generating high-accuracy constraint sets while reducing the number of queries; 2) developing a clustering algorithm tailored for utilizing the LLM-generated constraints to improve clustering accuracy.

\subsection{Algorithmic Framework}

In this work, we utilize LLMs to assist in generating both cannot-link and must-link constraints. After embedding text into data points, our method consists of two main stages: 1) constraint generation: we select a set of candidate data points and employ LLMs to assess their relationship, generating ML/CL constraints; 2) constrained clustering: we perform clustering on the data points while incorporating the ML/CL constraint sets generated among the inspected data points. Notably, the selection algorithms used in Stage 1 are tuned to better support the subsequent Stage 2, ultimately yielding clustering results with higher accuracy.

\subsection{Constraint Generation with LLM}
\label{subsec:Query_effi}

By setting candidate points, we propose a novel LLM-generated constraint algorithm aimed at improving query efficiency while maintaining the quality of constrained clustering. To this end, our candidate query selection algorithm is guided by two key factors: the correctness of the constraints generated by the LLMs and the effectiveness of the queries formulated to assist the LLMs. We provide the key idea of this stage below and further details on the algorithm can be found in the full version.

To enhance query efficiency, we leverage LLMs to generate constraint sets, thereby enhancing the effectiveness of the queries instead of using pairwise constraints~\cite{guo2024efficient}. Specifically, we propose two methods to separately construct candidate constraint sets for the must‑links and cannot‑links, ensuring that each set contains at least two ($\geq 2$) points per query.  In addition, following previous works~\cite{basu2004active,mallapragada2008active,viswanathan2024large}, we select candidate points based on distances.

\paragraph{Must-Links Constraint Set.} To collect candidate points for the must-link constraints, we employ an algorithm based on coresets~\cite{har2004coresets} to partition the dataset into representative subsets. 
This approach prevents the selection of points that differ significantly along any single dimension and ensures that each candidate ML set contains mutually similar points. Each sampled subset is mapped to its text and then passed to an LLM, which returns groups of texts. The corresponding points of these text groups are then used to construct the ML sets.
In particular, let $X_c = \{x_1, \ldots, x_m\} \subseteq S$ denote a selected must-link candidate point set obtained by the selection method. We map it to the corresponding text set $T_c = \pi^{-1}(X_c) = \{t_1, \ldots, t_m\}$, where $\pi^{-1}$ is the inverse mapping function from the text to its associated point. 
The final ML constraint sets are then constructed according to the LLM output as
\[
LLM_{ML}(T_c) = \big\{\{t_1, t_2, t_3\}, \ldots, \{t_m\}\big\}.
\]
We select each subgroup containing more than one text to form an ML set. 
For example, given $T = \{t_1, t_2, t_3\}$, we use the mapping $\pi$ to obtain the corresponding point set $X = \pi(T) = \{x_1, x_2, x_3\}$, which then defines the ML constraint.

\textit{Hard and Soft ML Constraints.} 
To effectively leverage LLM-generated constraints in clustering, we assign confidence values to the must-link constraints as indicated by the LLMs’ feedback. We separately compute two thresholds for the pairwise and multi-point constraints. The key steps are as follows: 1) According to the coreset algorithm, the points are divided into multiple grid levels $r_j$, defined as $r_j=(1+\varepsilon)^j\cdot\sqrt{cost_{kc}/10n\delta}$, where $cost_{kc}$ denotes the cost of the $k$-center problem solved by the min-max algorithm~\cite{gonzalez1985clustering} and $\varepsilon$ is a small constant set to $0.1$;
2) At each grid width level, we compute the ordering of the pairwise distances (or set diameters) as $\Psi=\{\max_{x_1,x_2\in X} d(x_1,x_2)\, \vert\, X\in S_{r_j}\}$, where $S_{r_j}$ denotes the family of ML sets whose inter-point distances correspond to the $j$th grid width $r_j$; 
3) We perform a binary search over $\Psi$ to identify a desirable distance threshold $\psi \in \Psi$; 
4) For each candidate threshold $\psi$, we query the LLMs $\alpha$ times ($\alpha = 5$ for pairwise constraints and $\alpha = 10$ for set-based constraints). If all responses are considered consistent, $\psi$ is designated as the maximum allowable diameter for a hard must-link constraint of the corresponding constraint type. In our clustering algorithm, we handle hard and soft ML constraints separately.

\paragraph{Cannot-Links Constraint Set.}On the other hand, to collect candidate points for the cannot-link constraints, we uniformly randomly select from the set of uncovered data points whose distance from the current candidate set exceeds a threshold $r = cost_{kc}$, where $cost_{kc}$ is the cost of $k$-center computed by the min-max algorithm~\cite{gonzalez1985clustering} to limit the size of the CL constraint set ($\leq k$) regarding its $2$-approximation ratio guarantee. This selection strategy is designed to facilitate more accurate judgments by LLMs. To ensure correctness, each sampled point $q$ is evaluated by an LLM to determine whether it should be placed in a cannot-link set. We denote the LLM's response as $LLM_{CL}(T_Y,t_q)$, where $Y$ is the current CL set, $T_Y = \pi(Y)$ is the set of texts mapped from $Y$ via the mapping function $\pi$, $q$ is a candidate point and $t_q = \pi(q)$. If $LLM_{CL}(T_Y,t_q)$ returns ``None'', we append $q$ to $Y$. Otherwise, we continue to seek the next point. The process continues until the size of $Y$ reaches $k$, or no new point $q$ can be found beyond the radius $r$ from $Y$, at which point the algorithm terminates and begins constructing a new CL set.

\begin{algorithm}[!t]
    \caption{\small{ML-Constrained Clustering with Penalty.}}
    \label{alg:ML_penalty_clustering_confi}
    \KwIn{A dataset $S$ with a family of ML sets $\mathcal{X}$ including the hard ML constraint sets $\mathcal{X}_h$.}
    \KwOut{The assignment result of ML sets $\mathcal{X}$ and center set $C$.}
    Set $C\gets$ $k$-means++~\cite{arthur2007kmeanspp}, $P \gets \{P_i\}_{i=1}^k$\;
    \tcc{Initialization center}
    \For{each ML set $X \in \mathcal{X}_{h}$}{
        Set the mass center $\bar{X}$ with $|X|$ to represent the set $X$\;
    }
    \tcc{Assignment step}
    \For{each ML set $X \in \mathcal{X}\setminus \mathcal{X}_{h}$}{
        \For{each point $x \in X$}{
            Assign $x$ to its nearest center $c_i \in C$\;
            Set partition $P_{i} \leftarrow P_{i} \cup x$\;
        }
    }
    \For{each $P_{i} \in P$ with the largest size $|P_i|$ }{
        Set the mass center $\bar{P_i}$ with $|P_i|$ to represent the partition$P_i$\tcp*{similar to $P_j$}
        \For{each $P_{j} \in P\setminus P_{i}$ with the largest size $|P_j|$}{
            Set $c_{ij}$ as the nearest center of the merged mass center of $\overline{P_{i}\cup P_{j}}$\;
            \If{$(w_{m} + d(\bar{P_{j}}, c_j)) \cdot|P_{j}| + (w_{m} +d(\bar{P_{i}}, c_i))\cdot$ $ |P_i|> \sum_{p \in P_{i}\cup P_{j}}d(p,c_{ij})$}{
                Set $P_{i} \leftarrow P_{i}\cup P_{j}$\;
            }
        }
    }
    Return $\mathcal{X}' \gets P$ and $C$.
\end{algorithm}

\subsection{Clustering with LLM-generated Constraints}
\label{subsec:alg_llm}
Our LLM-generated constraints differ from traditional constraints in two key aspects: 1) they are constructed as sets rather than pairwise, and 2) they may contain erroneous constraints in the LLM's output.
In this clustering algorithm, we begin by initializing the clusters using the hard must-link (ML) constraints with the clustering $k$-means++. For clarity, we next introduce penalty-based constrained clustering algorithms that incorporate CL and ML constraints, respectively. Next, we present a unified framework that can handle both types of constraints. Finally, the algorithm needs to update the centers using the assignment results and repeat the constrained assignment steps until convergence.

\subsubsection{Cluster Initialization\\}
\label{subsubsec:clust_init}

Similar to previous studies~\cite{baumann2024algorithm}, we represent hard must-link constraints by their mass center and incorporate them during the seeding step based on the $k$-means++ initialization strategy~\cite{arthur2007kmeanspp}. The $k$-Means++ algorithm achieves better clustering accuracy by adding the seeding step to the traditional $k$-means algorithm, which means the $k$-means algorithm is sensitive to the improvement of the initial center selection steps. Thus, to keep the advantage, we divide the hard and soft ML constraints during the generation stage and utilize the hard  ML constraints (with high correctness) to join the initial centers selection step. The detailed algorithm is provided in Steps 1-3 of Alg.~\ref{alg:ML_penalty_clustering_confi}.

\subsubsection{ML Clustering with Penalty\\}
 As described in the above subsection, we propose distinguishing high-accuracy constraints as hard constraints to serve the initialization step, while treating the remainder as soft constraints. The hard ML constraints will be assigned following their representative point. The key ideas of our method for handling soft constraints are summarized below, with the full algorithm detailed in Alg.~\ref{alg:ML_penalty_clustering_confi}. 

For each soft ML set $X \in \mathcal{X}\setminus X_h$, our approach first assigns the points in $X$ to their partitions in $P$. To be specific, each point is assigned to its nearest center and divides $X$ into subsets $P_{i} \in P$, which is represented by its centroid $\bar{P_{i}}$. 
Next, we iteratively select the two partitions with the largest diameters, denoted as $P_i$ and $P_j$, and compare the cost of merging them (obtained by assigning the merged group to the nearest center to its mass center) with the cost of keeping the original assignments and the associated penalties. If the merge yields a lower overall cost, we combine $P_i$ and $P_j$ and update the partition set $X$ accordingly.
This process continues until no further merges within $X$ can be merged.
\begin{lem}\label{lem:runtime_ml}
The running time of Alg.~\ref{alg:ML_penalty_clustering_confi} is $O(nk^2)$.
\end{lem}
\begin{proof}
Firstly, the initialization step of Alg.~\ref{alg:ML_penalty_clustering_confi} takes $O(n)$ time to process the ML sets $\mathcal{X}$, followed by $O(nk^2)$ for $k$-means++ to obtain the center set.
Then, the assignment step takes $O(nk)$ time to enforce the ML constraints on the partitions, and on all partitions, we spend a total of $O(nk^2)$ time comparing the cost of merging versus splitting each pair of partitions. So the total runtime of Alg.~\ref{alg:ML_penalty_clustering_confi} is $O(nk^2)$.      
\end{proof}

\begin{algorithm}[!t]
    \caption{{Local Search for CL-Constrained Clustering with Penalty.}}
    \label{alg:CL_penalty_clustering_local}
    \KwIn{A dataset $S$ with a family of CL sets $\mathcal{Y}$ and the center set $C$.}
    \KwOut{The assignment result $\mathcal{A}$ of CL sets $\mathcal{Y}$.}
    \For{each CL set $Y \in \mathcal{Y}$}{
    \While{True}{
        Construct the auxiliary bipartite graph $G(C,Y;E)$ regarding the center set $C$ and the CL set $Y$\;
        Set $M \leftarrow$ Computing the maximum-sum matching on graph $G$\;
        \For{each edge $ e(y,c) \in M$}{
          Compute the maximum-sum matching regarding graph $G' = G(C,Y\setminus y)$ and set the matching as  $M' $\;
         Set $num_{y} \leftarrow 1$\;
          \For{each  $y' \in Y\setminus y$}{
          \lIf{$M(y') \neq M'(y')$}{
         $num_{y}$++
         }}
          Set  $g_y \leftarrow \sum_{e(y,c)\in M} d(y,c) - \sum_{e(y,c')\in M'} d(y,c') - d(y,c(y))$\;
        }
        \eIf{$\max_{y \in Y} g_y <  num_{y} \cdot w_{cl}$}{
           Assign CL set $Y$ to $\mathcal{A}$ by matching $M$\;
           \textbf{break}\;
        }{
            Assign $y$ to $A_y$ with its nearest center $c(y)$\;
            Set $Y\leftarrow Y\setminus y$\;
        }
    }  
    }

\end{algorithm}

\subsubsection{CL Clustering with Penalty\\}
 To efficiently utilize CL constraint sets, we use the maximum-sum matching approach for $k$-means clustering. To further reduce the impact of the false relationships, we utilize a penalty term to correct the clustering assignments by local search. The key steps are outlined below, with the algorithm detailed in Alg.~\ref{alg:CL_penalty_clustering_local}.

\begin{algorithm}[t]
    \caption{\small LSCK-HC:~ML/CL-Constrained Clustering.}
    \label{alg:CL_ML_penalty}
    \KwIn{A dataset $S$ of  size $n$, a family of CL sets $\mathcal{Y}$ and a family of ML sets $\mathcal{X}$.}
    \KwOut{A set of clusters $\mathcal{A}$ regarding center set $C$.}

   Set $\mathcal{X'}$ and $C \leftarrow$ CALL Alg.~\ref{alg:ML_penalty_clustering_confi} regarding the ML constraint sets $\mathcal{X}$\;
    \For{$X'\in \mathcal{X'}$}{
    Compute the mass center $\bar{X'}$ with weight $|X'|$ to represent the set $X'$\;
    Set $S \gets S \cup \bar{X'}\setminus X'$\;
    }
    Set $\mathcal{A} \gets$ CALL Alg.~\ref{alg:CL_penalty_clustering_local} regarding datasets $S$ with the family of CL constraint sets $\mathcal{Y}$ and $C$\;
    \For{each unassigned point $s \in S$}{
        Set $S' \gets s$\;
        \lIf{$s\in X\in \mathcal{X}$}{
        Set $S' \gets X$
        }
        Assign $S'$ to its nearest center $c(S') \in C$\; Set $A_{c(S')} \gets A_{c(S')}\cup S'$\;
        }
    Return $\mathcal{A}$.
\end{algorithm}

\paragraph{Auxiliary Bipartite Graph}

Given the center set $C$ and each CL set $Y \in \mathcal{Y}$, we construct an auxiliary bipartite graph $G(Y,C;E)$, where the weight of each edge $e \in E$ is the negative of the distance between its endpoints. The formal definition is as follows.

\begin{defn}\label{def:AG}
Given a center set $C$ and a CL set $Y$, the auxiliary bipartite graph $G(Y,C;E)$ is defined on the vertex set $Y\cup C$, and $E$ is the set of edges between $Y$ and $C$.
\end{defn}

\paragraph{Minimum Weight Perfect Matching}
According to the graph $G$ defined by Def.~\ref{def:AG}, we compute a maximum-sum matching $M$ to identify centers for the points in each CL set $Y \in \mathcal{Y}$. For a given CL set $Y$, let $M$ denote the matching between the center set $C$ and $Y$ as below. 
\begin{defn}
\label{defn:minmax-matching}
Given an auxiliary bipartite graph $G=(C,Y;E)$, the matching $M_G\subseteq E$ is defined as a minimum weight matching if and only if: 1) $M_G$ is a one-sided perfect matching with $|Y|\leq |C|$; 2) $\sum_{e(y,c)\in M_G}d(y,c)$ attains minimum, where $y \in Y$ and $c \in C$. 
\end{defn}

Based on the above definition, the algorithm proceeds as follows. For each matching $M$: 1) Remove the point $y$ with the largest distance, i.e., $y = \arg\max_{(y,c)\in M} d(y,c)$; 2) Let $M'$ be the maximum-sum matching between $Y\setminus\{y\}$ and $C$; 3) By comparing $M$ and $M'$, define $\mathit{num}_y$ that is the number of points in $Y\setminus\{y\}$ whose center assignments change and $g_y$ is the corresponding decrease in total cost; 4) If $\max_{y\in Y} g_y < \mathit{num}_y \times w_{cl}$ (where $w_{cl}$ is the penalty), the entire CL set will be assigned according to $M$; Otherwise, let $y = \arg\max_{y\in Y} g_y$, reassign $y$ to its nearest center $c(y)$, update $Y \leftarrow Y\setminus\{y\}$, and repeat steps 1-4 until for the chosen $y$ we have $g_y \le \mathit{num}_y \times w_{cl}$.

\begin{lem}\label{lem:runtime_cl}
The running time of Alg.~\ref{alg:CL_penalty_clustering_local} is $O(k^\frac{9}{2}+nk^4)$.
\end{lem}
\begin{proof}
In steps 3-4 of the algorithm, each matching requires $O(|E(G)|+|V(G)|^\frac{3}{2})=O(k^{\frac{3}{2}}+ k\cdot|Y|)$ time to compute the maximum matching in $G$ using the algorithm by~\cite{van2020bipartite}. Similarly, the matching cost $O(k^{\frac{3}{2}}+ k\cdot|Y|)$ in steps 6-9. Thus, when we compare these assignment results, if the cost reduction is always greater than the penalty, we spend $O(\sum_{Y\in \mathcal{Y}}(k^{\frac{3}{2}}+ k\cdot|Y|)^2\times k) = O(k^\frac{9}{2}+k^4\cdot n)$ time to obtain all matching results. 
\end{proof}

\subsubsection{Clustering with Both ML and CL Constraints\\}
By combining Algs.~\ref{alg:ML_penalty_clustering_confi} and~\ref{alg:CL_penalty_clustering_local}, we present our overall method for CL and ML constrained clustering in Alg.~\ref{alg:CL_ML_penalty}. In this algorithm, we first apply Alg.~\ref{alg:ML_penalty_clustering_confi} to assign the ML constraint sets. Considering a point that belongs to both CL and ML relationships, it is crucial to ensure that the assignments respect the centers determined by the ML set in Step 4 of Alg.~\ref{alg:CL_ML_penalty} when addressing the CL constraints.

\begin{lem}\label{lem:runtime_cl_ml}
The running time of Alg.~\ref{alg:CL_ML_penalty} is $O((k^\frac{1}{2}+n)\cdot k^4)$.
\end{lem}

\section{Experiments}
\label{sec:exp}
In this section, we report the proposed clustering method using LLM-generated constraints. We conduct an extensive comparison with four baseline methods under these constraints. Additionally, we evaluate our approach in terms of query accuracy and query time reduction. We then demonstrate the performance of our algorithm (LSCK-HC) across various embeddings and models. 
Finally, for constraint generation, we analyze the individual contributions of cannot-link and must-link constraints. 

\subsection{Experiment Setup}
\label{sec:exp_setup}
\textbf{Datasets.} Following the related previous works~\cite{zhang2023clusterllm,viswanathan2024large}, we compare the performance of our algorithms with baselines on the text clustering tasks on five datasets: tweet, banking77, clinc (I/D) and GoEmo. 

\begin{figure*}
    \centering
    \includegraphics[width=\linewidth]{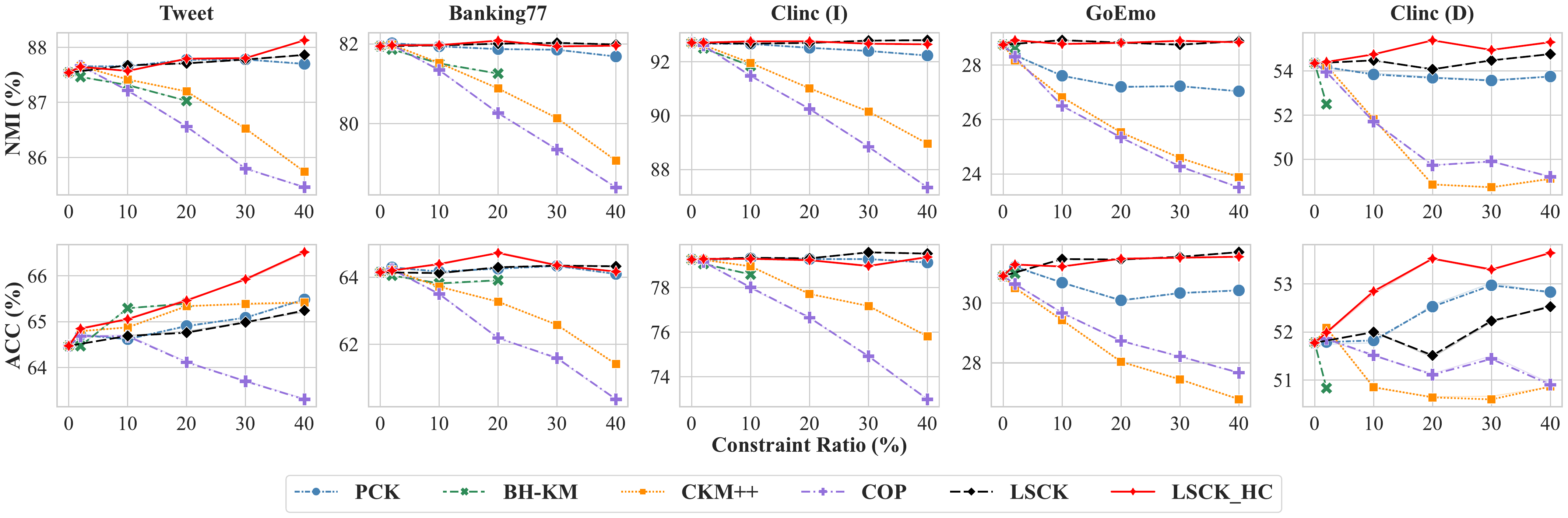}
    \caption{Comparison of average clustering results on datasets across different constrained instance ratios.}
    \label{fig:clustering_res}
\end{figure*}

\noindent\textbf{Baselines.}
We compare the performance with the closest baseline~\cite{viswanathan2024large}, called \textbf{FSC}, which uses the min-max method~\cite{mallapragada2008active} to select the query nodes and the PCK-means~\cite{basu2004active} clustering algorithm to cluster the data. For the clustering algorithm, we utilize the \textbf{$k$-means++}~\cite{arthur2007kmeanspp} as one of the baselines. 
In addition, we list the different popular constrained clustering algorithms based on $k$-means clustering below: 1) \textbf{COP}~\cite{wagstaff2001constrained}: It is the first algorithm designed to address the constrained clustering problem. 2) {\textbf{PCK}~\cite{basu2004active}:} It is a constrained clustering algorithm suitable for large data sets with sparse, high-dimensional data~\cite{deng2024a3s}. 3) \textbf{BH-KM}~\cite{baumann2022k}: A recent clustering algorithm with soft constraints, which uses a mixed-integer programming formulation, struggles to scale and tends to stop when faced with a large number of constraints;
4) \textbf{CKM++}~\cite{jia2023efficient}: It provides an algorithm to handle the hard constraint sets instead of the pairwise constraints. 
In addition, we set our clustering algorithm without any hard constraints, called \textbf{LSCK}.

\noindent\textbf{Metrics.}
According to the previous works~\cite{zhang2023clusterllm,viswanathan2024large,huang2024text,zhang2024unraveling}, we report the clustering performance by four metrics: accuracy (ACC) calculated after Hungarian alignment~\cite{kuhn1955hungarian}, and Normalized Mutual Information (NMI) calculates mutual information between two assignments, Rand Index (RI) and Adjusted Rand Index (ARI)  to quantify the similarity between the predicted clusters and the ground‑truth labels.

\noindent\textbf{Experimental Settings.}
The main experimental results evaluate the clustering with the constraints generated by LLMs. We directly apply these clusterings on extracted embeddings from Instructor-large~\cite{su2023one} and E5~\cite{wang2023goal} with the same prompt for~\cite{zhang2023clusterllm}. All experiments were conducted on a Linux machine equipped with an NVIDIA A100 80GB GPU, a 32-core CPU, and 256 GB of RAM, running Python 3.10. Each reported value in this experiment is the average result over 10 runs for each setting.

\subsection{Clustering Results with Constraints}
\label{subsec:main_exp_result}

We report the clustering quality of each algorithm regarding the same constraint sets while varying the ratio of constrained instances. BH-KM does not report all of the constraint ratios due to the difficulty of handling a large number of them.

\paragraph{Combining CL and ML Constraints.} 
In Sec.~\ref{subsec:Query_effi}, we provide the query sampling methods for generating cannot-link and must-link constraints, respectively. To evaluate clustering quality, we combine both constraint types through this process: 1) Randomly choose a CL set $Y$ to include in the final constraint collection; 2) If any points in $Y$ are included in a must-link (ML) constraint set, add the relevant ML constraints.
Continue repeating steps 1 and 2 until the target constraint ratio is achieved.

In Fig.~\ref{fig:clustering_res}, we show the clustering results for all datasets with constraints alongside the baseline results. The main observations are summarized as follows:

Compared with the baselines, \textbf{LSCK-HC and LSCK consistently improve clustering performance over the unconstrained setting.} Across all datasets, our algorithms outperform the 0\% constraint clustering results. In particular, clustering accuracy increases by more than 2\% as the constraint ratio rises to 40\% on the Tweet and CLINC (D) datasets. This improvement surpasses that of the baselines, which we attribute to the effectiveness of local search in finding solutions.

Except for our algorithms, it is observed that PCK generally outperforms COP and CKM++ by including a penalty to handle the constraints as soft constraints in its cost function. However, the effectiveness of this penalty diminishes when the constraint ratio exceeds 10\% across multiple datasets (e.g., GoEmo, Clinc (I)), leading to a decline in clustering quality. In contrast, our algorithms leverage a local search method, achieving robust clustering results. 
\textbf{The experimental results are consistent with our design, which can effectively utilize LLM-generated constraints. }

Alg.~\ref{alg:CL_ML_penalty} achieves its best performance with a 20\% constraint ratio for most datasets.
Moreover, this case can be attributed to two main factors: (1) as the number of constraints increases, the proportion of erroneous constraints also rises; and (2) prior studies including PCK and CKM++~\cite{basu2004active,jia2023efficient} have shown that improvements in clustering quality tend to plateau as more constraints are added, even when all constraints are correct. However, Tweet dataset is an exception in our experiments. We believe this is due not only to the high accuracy of LLM-generated constraints for this dataset, but also to its inherent focus on clustering as the primary task.

\subsection{Effectiveness and Efficiency of Constraint Generation}
\label{subsec:constraints_quality}
 In this section, we examine our method described in Sec.~\ref{subsec:Query_effi} with the baseline approach~\cite{viswanathan2024large}. We adjust the skeleton structure to better serve the CL constraints and draw inspiration from the coreset technique to develop an ML sampling method. To evaluate constraint effectiveness, we compare both the number of LLM queries and the quality of the resulting constraints by Rand Index.

\textbf{Query Times.} 
As shown in Tab.~\ref{tab:constraints_quality_QT}, we compare our algorithms (see Sec.~\ref{sec:Methods}) with the baseline in terms of the number of queries across different constraint ratios and datasets. \textbf{Our approach demonstrates a substantial advantage, reducing the number of LLM queries by at least 20-fold for each ratio on all datasets.} In particular, we vary the constraint ratio from 2\% to 20\% and report the corresponding LLM query counts for CL and ML constraints. We attribute the improvements to the selection method of the candidate constraint sets. Moreover, our proposed selection strategy enables the generation of multiple constrained relationships from a single query.

\textbf{Constraints Quality.} 
We evaluate the quality of the LLM-generated constraints in Tab.~\ref{tab:constraints_quality_QT}, which demonstrates that our method consistently achieves higher accuracy than the FSC. In particular, for must-link (ML) constraints, our algorithm significantly improves correctness. For example, when the constraint ratio is set to 2\%, the accuracy of FSC remains below 50\%, whereas our method achieves over 96\%. This substantial gap is likely due to the baseline deriving ML constraints indirectly, using inference steps that were originally designed for identifying cannot-link (CL) constraints. In contrast, our method explicitly selects the high-confidence ML constraint set based on distance-based criteria. As the constraint ratio increases, the accuracy of FSC also improves, mainly because more ML constraints are generated from steps specifically intended for must-link identification. Nevertheless, our method consistently maintains a higher level of accuracy, outperforming FSC by more than 12\% across all settings.

\subsection{Ablation Studies}
As illustrated in Fig.~\ref{fig:diff_model} and Tab.~\ref{tab:diff_embedding}, we evaluate the algorithms under two different models and embeddings. Then, we measure the influence of the constraints' correctness on the clustering performance, and the separate contributions for CL and ML clustering can be found in the full version.
\begin{table}[t]
\centering
\scalebox{0.68}{
\begin{tabular}{l|c|cccccc}
\toprule
\multirow{2}{*}{\textbf{Datasets}} & \multirow{2}{*}{\textbf{Methods}} &
\multicolumn{2}{c}{\textbf{2\%}} & \multicolumn{2}{c}{\textbf{10\%}} & \multicolumn{2}{c}{\textbf{20\%}} \\
\cmidrule(lr){3-4} \cmidrule(lr){5-6} \cmidrule(lr){7-8}
& & RI & \#Query & RI & \#Query & RI & \#Query \\
\midrule[0.5pt]
\midrule[0.5pt]
\multirow{2}{*}{Bank77-ML} & FSC  & 9.92 & 5260 & 8.13 & 23420 & 11.40 & 26580 \\
& Our & 
\textbf{96.42} & \textbf{91} & 
\textbf{88.08} & \textbf{313} & 
\textbf{85.83} & \textbf{480} \\
\midrule[0.4pt]
\multirow{2}{*}{Bank77-CL} & FSC  & 99.55 & 4070 & 99.32 & 22620 & 99.28 & 37715 \\
& Our & 
\textbf{99.75} & \textbf{89} & 
\textbf{99.50} & \textbf{641} & 
\textbf{99.48} & \textbf{1637} \\
\midrule[0.7pt]
\multirow{2}{*}{CLINC-ML} & FSC  & 42.05 & 25095 & 69.20 & 26725 & 78.97 & 29225 \\
& Our & 
\textbf{96.25} & \textbf{91} & \textbf{94.29} & \textbf{467} & \textbf{91.67} & \textbf{927} \\
\midrule[0.4pt]
\multirow{2}{*}{CLINC-CL} & FSC  & \textbf{99.57} & 12175 & 99.49 & 36100 & 99.48 & 42135 \\
& Our & 
99.52 & \textbf{93} & 
\textbf{99.64} & \textbf{662} & 
\textbf{99.59} & \textbf{1751} \\
\midrule[0.7pt]
\multirow{2}{*}{Tweet-ML} & FSC  & 50.00 & 13070 & 74.35 & 25800 & 83.82 & 27425 \\
& Our& 
\textbf{100.00} & \textbf{99} & 
\textbf{100.00} & \textbf{258} & 
\textbf{99.21} & \textbf{539} \\
\midrule[0.4pt]
\multirow{2}{*}{Tweet-CL} & FSC  & \textbf{99.74} & 5330 & 99.19 & 27440 & 99.17 & 31485 \\
& Our & 
99.66 & \textbf{111} & 
\textbf{99.67} & \textbf{323} & 
\textbf{99.56} & \textbf{821} \\
\bottomrule
\end{tabular}
}
\caption{Comparison of the LLM-generated constraint quality (RI) and the number of queries (\#Query) across different constrained instance ratios.}
\label{tab:constraints_quality_QT}
\end{table}
\paragraph{Robustness for Different Models and Embeddings.} 
In Fig.~\ref{fig:diff_model}, we compare the performance of two large language models (i.e., DeepSeek R1 and V3) on the Tweet and Banking77 datasets. First, we observe that the quality of the generated constraints varies with the chosen LLM, and those differences in constraint quality directly affect the resulting clustering performance. According to this, we provide the experimental results on how constraint quality affects performance. Next, we find that when evaluating clustering results with the same model, our algorithm outperforms PCK, aligning with the findings discussed in the above section. Thirdly, we show the improvement for our framework (LSCK-HC) over the baseline (FSC) while maintaining a consistent model (V3). For both of the datasets, we discovered that our algorithms exceed baseline performance. We attribute this advantage to the design of our selection strategy and clustering algorithm.

As shown in Tab.~\ref{tab:diff_embedding}, we compare clustering results for the Instructor‑large~\cite{su2023one} and E5~\cite{wang2023goal} embeddings, using the same language model and evaluated on three metrics. Our algorithm outperforms PCK on both embeddings, and PCK itself yields higher clustering performance with Instructor‑large than with E5, which aligns with previous evaluations~\cite{zhang2023clusterllm}. Under constrained clustering, our method delivers the greatest boost on the smaller E5, improving ARI by nearly 10\%. We attribute this to the lower baseline clustering quality of the smaller model, which allows the added constraints to have a stronger corrective effect.

\begin{figure}
    \includegraphics[width=\linewidth]{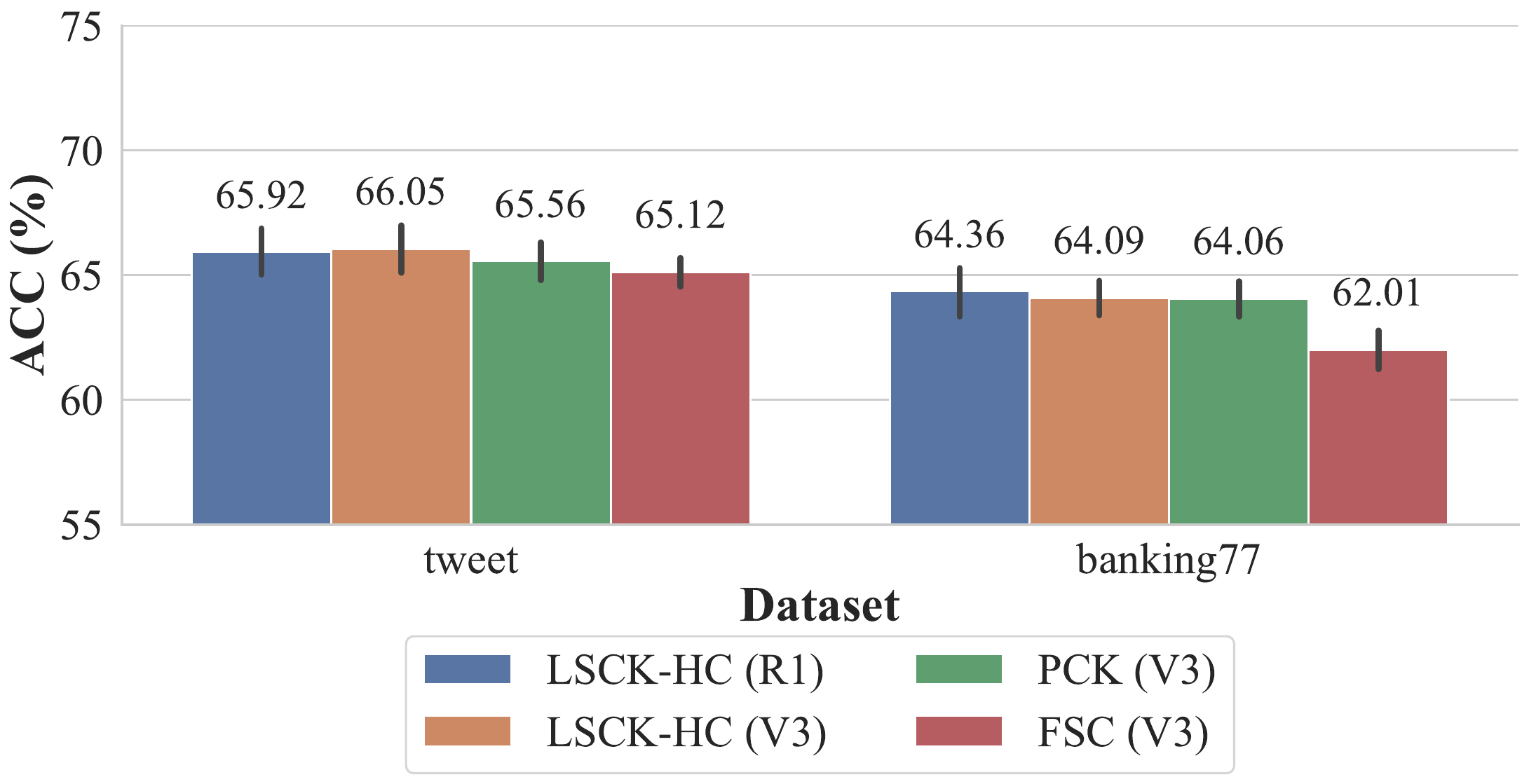}
    \caption{Comparison of clustering accuracy (ACC) across different large language models and algorithms with a $30\%$ constrained-instance ratio. Note that FSC in~\citet{viswanathan2024large} uses its own constraint‑generation algorithm and applies PCK for clustering.}
    \label{fig:diff_model}
\end{figure}

\begin{table}[t]
\centering
\scalebox{0.69}{
\begin{tabular}{l|cc|cc|cc}
\toprule
\multirow{2.5}{*}{Embedding}  
  & \multicolumn{2}{c|}{ACC} 
  & \multicolumn{2}{c|}{NMI} 
  & \multicolumn{2}{c}{ARI} \\
\cmidrule(lr){2-3}  \cmidrule(lr){4-5} \cmidrule(lr){6-7} 
  & PCK         & LSCK‑HC          & PCK         & LSCK‑HC          & PCK         & LSCK‑HC          \\
\midrule
Instructor‑large 
  & 65.08       & \underline{65.92} & 87.78       & \underline{87.80} & 55.30       & 55.89 \\ 
E5               
  & 60.18       & 62.90 & 85.36       & 85.86 & 48.68       & \underline{58.00} \\ 
\bottomrule
\end{tabular}
}
\caption{Comparison of clustering results on the Tweet dataset at a 30\% constrained-instance ratio across different embeddings. Underlined values indicate the best performance for each metric.}
\label{tab:diff_embedding}
\end{table}

\section{Conclusion}
\label{sec:conclusion}
In this paper, we present a text clustering algorithm using must-link and cannot-link constraints. Leveraging LLMs to generate constraints, we aim to improve clustering quality through an efficient constraint generation method while reducing resource consumption by treating constraints as sets. We further tune the constrained clustering algorithm to handle these LLM-generated constraints, including both hard and soft ML constraints as well as CL constraints. Moreover, we introduce penalties to mitigate the impact of false constraints and enhance clustering performance by local search. We evaluate our results on five short-text datasets, demonstrating that our method compares favorably in both clustering quality and the cost of constraint generation.

\section*{Acknowledgments}

This work is supported by National Natural Science Foundation of China (No. 12271098) and Key Project of the Natural Science Foundation of Fujian Province  (No. 2025J02011). 

\bibliography{aaai2026}

@inproceedings{zhang2023clusterllm,
  author    = {Yuwei Zhang and Zihan Wang and Jingbo Shang},
  title     = {ClusterLLM: Large Language Models as a Guide for Text Clustering},
  booktitle = {Proceedings of the 2023 Conference on Empirical Methods in Natural Language Processing (EMNLP)},
  pages     = {13903--13920},
  year      = {2023}
}

@article{viswanathan2024large,
  author    = {Vijay Viswanathan and Kiril Gashteovski and Carolin Lawrence and Tongshuang Wu and Graham Neubig},
  title     = {Large Language Models Enable Few-Shot Clustering},
  journal   = {Transactions of the Association for Computational Linguistics (TACL)},
  volume    = {12},
  pages     = {321--333},
  year      = {2024}
}

@inproceedings{wagstaff2001constrained,
  author    = {Kiri Wagstaff and Claire Cardie and Seth Rogers and Stefan Schroedl},
  title     = {Constrained $k$-Means Clustering with Background Knowledge},
  booktitle = {Proceedings of the 18th International Conference on Machine Learning (ICML)},
  pages     = {577--584},
  year      = {2001}
}

@inproceedings{wagstaff2000clustering,
  author    = {Kiri Wagstaff and Claire Cardie},
  title     = {Clustering with Instance-Level Constraints},
  booktitle = {Proceedings of the 17th International Conference on Machine Learning (ICML)},
  pages     = {1103--1110},
  year      = {2000}
}

@inproceedings{basu2002semi,
  author    = {Sugato Basu and Arindam Banerjee and Raymond J. Mooney},
  title     = {Semi-Supervised Clustering by Seeding},
  booktitle = {Proceedings of the 19th International Conference on Machine Learning (ICML)},
  pages     = {27--34},
  year      = {2002}
}

@inproceedings{basu2004active,
  author    = {Sugato Basu and Arindam Banerjee and Raymond J. Mooney},
  title     = {Active Semi-Supervision for Pairwise Constrained Clustering},
  booktitle = {Proceedings of the 2004 SIAM International Conference on Data Mining (SDM)},
  pages     = {333--344},
  year      = {2004},
  organization = {SIAM}
}

@inproceedings{arthur2007kmeanspp,
  author    = {David Arthur and Sergei Vassilvitskii},
  title     = {$k$-Means++: The Advantages of Careful Seeding},
  booktitle = {Proceedings of the 18th Annual ACM-SIAM Symposium on Discrete Algorithms (SODA)},
  pages     = {1027--1035},
  year      = {2007}
}

@inproceedings{guo2024efficient,
  author    = {Longkun Guo and Chaoqi Jia and Kewen Liao and Zhigang Lu and Minhui Xue},
  title     = {Efficient Constrained $k$-Center Clustering with Background Knowledge},
  booktitle = {Proceedings of the AAAI Conference on Artificial Intelligence (AAAI)},
  volume    = {38},
  pages     = {20709--20717},
  year      = {2024}
}

@inproceedings{deng2024a3s,
  title={A3S: A General Active Clustering Method With Pairwise Constraints},
  author={Deng, Xun and Liu, Junlong and Zhong, Han and Feng, Fuli and Shen, Chen and He, Xiangnan and Ye, Jieping and Wang, Zheng},
  booktitle={Proceedings of the 41st International Conference on Machine Learning (ICML)},
  pages={10488--10505},
  year={2024}
}

@article{baumann2024algorithm,
  title={An Algorithm for Clustering With Confidence-Based Must-Link and Cannot-Link Constraints},
  author={Baumann, Philipp and Hochbaum, Dorit S},
  journal={INFORMS Journal on Computing},
  year={2024},
  publisher={INFORMS}
}

@article{jia2023efficient,
  title={Efficient Algorithm for the $k$-Means Problem with Must-Link and Cannot-Link Constraints},
  author={Jia, Chaoqi and Guo, Longkun and Liao, Kewen and Lu, Zhigang},
  journal={Tsinghua Science and Technology},
  volume={28},
  number={6},
  pages={1050--1062},
  year={2023},
  publisher={TUP}
}

@article{yu2018semi,
  title={Semi-Supervised Ensemble Clustering Based on Selected Constraint Projection},
  author={Yu, Zhiwen and Luo, Peinan and Liu, Jiming and Wong, Hau-San and You, Jane and Han, Guoqiang and Zhang, Jun},
  journal={IEEE Transactions on Knowledge and Data Engineering},
  volume={30},
  number={12},
  pages={2394--2407},
  year={2018},
  publisher={IEEE}
}

@inproceedings{mallapragada2008active,
  title={Active Query Selection for Semi-Supervised Clustering},
  author={Mallapragada, Pavan Kumar and Jin, Rong and Jain, Anil K},
  booktitle={ICPR 2008 19th International Conference on Pattern Recognition (ICPR)},
  year={2008},
  organization={IEEE Computer Society}
}

@article{xiong2013active,
  title={Active Learning of Constraints for Semi-Supervised Clustering},
  author={Xiong, Sicheng and Azimi, Javad and Fern, Xiaoli Z},
  journal={IEEE Transactions on Knowledge and Data Engineering},
  volume={26},
  number={1},
  pages={43--54},
  year={2013},
  publisher={IEEE}
}

@inproceedings{davidson2005clustering,
  title={Clustering With Constraints: Feasibility Issues and the $k$-Means Algorithm},
  author={Davidson, Ian and Ravi, SS},
  booktitle={Proceedings of the 2005 SIAM International Conference on Data Mining (SDM)},
  pages={138--149},
  year={2005},
  organization={SIAM}
}

@inproceedings{pelleg2007k,
  title={$k$-Means With Large and Noisy Constraint Sets},
  author={Pelleg, Dan and Baras, Dorit},
  booktitle={European Conference on Machine Learning (ECML)},
  pages={674--682},
  year={2007},
  organization={Springer}
}

@inproceedings{ganji2016lagrangian,
  title={Lagrangian Constrained Clustering},
  author={Ganji, Mohadeseh and Bailey, James and Stuckey, Peter J},
  booktitle={Proceedings of the 2016 SIAM International Conference on Data Mining (SDM)},
  pages={288--296},
  year={2016},
  organization={SIAM}
}

@inproceedings{le2018binary,
  title={A Binary Optimization Approach for Constrained $k$-Means Clustering},
  author={Le, Huu M and Eriksson, Anders and Do, Thanh-Toan and Milford, Michael},
  booktitle={Asian Conference on Computer Vision (ACCV)},
  pages={383--398},
  year={2018},
  organization={Springer}
}

@inproceedings{baumann2020binary,
  title={A Binary Linear Programming-Based $k$-Means Algorithm for Clustering With Must-Link and Cannot-Link Constraints},
  author={Baumann, Philipp},
  booktitle={2020 IEEE International Conference on Industrial Engineering and Engineering Management (IEEM)},
  pages={324--328},
  year={2020},
  organization={IEEE}
}

@inproceedings{baumann2022k,
  title={A $k$-Means Algorithm for Clustering with Soft Must-link and Cannot-link Constraints},
  author={Baumann, Philipp and Hochbaum, Dorit},
  booktitle={the 11th International Conference on Pattern Recognition Applications and Methods (ICPRAM)},
  year={2022}
}

@article{gonzalez1985clustering,
  title={Clustering To Minimize the Maximum Intercluster Distance},
  author={Gonzalez, Teofilo F},
  journal={Theoretical Computer Science},
  volume={38},
  pages={293--306},
  year={1985},
  publisher={Elsevier}
}

@inproceedings{har2004coresets,
  title={On Coresets for $k$-Means and $k$-Median Clustering},
  author={Har-Peled, Sariel and Mazumdar, Soham},
  booktitle={Proceedings of the Thirty-Sixth Annual ACM Symposium on Theory of Computing (STOC)},
  pages={291--300},
  year={2004}
}

@article{kuhn1955hungarian,
  title={The Hungarian Method for the Assignment Problem},
  author={Kuhn, Harold W},
  journal={Naval Research Logistics Quarterly},
  volume={2},
  number={1-2},
  pages={83--97},
  year={1955},
  publisher={Wiley Online Library}
}

@article{huang2024text,
  title={Text Clustering As Classification With LLMs},
  author={Huang, Chen and He, Guoxiu},
  journal={ArXiv Preprint ArXiv:2410.00927},
  year={2024}
}

@inproceedings{su2023one,
  title={One Embedder, Any Task: Instruction-Finetuned Text Embeddings},
  author={Su, Hongjin and Shi, Weijia and Kasai, Jungo and Wang, Yizhong and Hu, Yushi and Ostendorf, Mari and Yih, Wen-tau and Smith, Noah A and Zettlemoyer, Luke and Yu, Tao},
  booktitle={Findings of the Association for Computational Linguistics (ACL)},
  pages={1102--1121},
  year={2023}
}

@article{lloyd1982least,
  title={Least Squares Quantization in PCM},
  author={Lloyd, Stuart},
  journal={IEEE Transactions on Information Theory},
  volume={28},
  number={2},
  pages={129--137},
  year={1982},
  publisher={IEEE}
}

@article{ahmed2022short,
  title={Short Text Clustering Algorithms, Application and Challenges: A Survey},
  author={Ahmed, Majid Hameed and Tiun, Sabrina and Omar, Nazlia and Sani, Nor Samsiah},
  journal={Applied Sciences},
  volume={13},
  number={1},
  pages={342},
  year={2022},
  publisher={MDPI}
}

@article{bae2020interactive,
  title={Interactive Clustering: A Comprehensive Review},
  author={Bae, Juhee and Helldin, Tove and Riveiro, Maria and Nowaczyk, S{\l}awomir and Bouguelia, Mohamed-Rafik and Falkman, G{\"o}ran},
  journal={ACM Computing Surveys (CSUR)},
  volume={53},
  number={1},
  pages={1--39},
  year={2020},
  publisher={ACM New York, NY, USA}
}

@inproceedings{fu2024acdm,
  title={ACDM: An Effective and Scalable Active Clustering with Pairwise Constraint},
  author={Fu, Xun and Xie, Wen-Bo and Chen, Bin and Deng, Tao and Zou, Tian and Wang, Xin},
  booktitle={Proceedings of the 33rd ACM International Conference on Information and Knowledge Management (CIKM)},
  pages={643--652},
  year={2024}
}

@article{li2019ascent,
  title={Ascent: Active Supervision for Semi-Supervised Learning},
  author={Li, Yanchao and Wang, Yongli and Yu, Dong-Jun and Ye, Ning and Hu, Peng and Zhao, Ruxin},
  journal={IEEE Transactions on Knowledge and Data Engineering},
  volume={32},
  number={5},
  pages={868--882},
  year={2019},
  publisher={IEEE}
}

@article{xiong2016active,
  title={Active Clustering With Model-Based Uncertainty Reduction},
  author={Xiong, Caiming and Johnson, David M and Corso, Jason J},
  journal={IEEE Transactions on Pattern Analysis and Machine Intelligence},
  volume={39},
  number={1},
  pages={5--17},
  year={2016},
  publisher={IEEE}
}

@inproceedings{davidson2005agglomerative,
  title={Agglomerative Hierarchical Clustering With Constraints: Theoretical and Empirical Results},
  author={Davidson, Ian and Ravi, SS},
  booktitle={European Conference on Principles of Data Mining and Knowledge Discovery (PKDD)},
  pages={59--70},
  year={2005},
  organization={Springer}
}

@article{cai2023review,
  title={A Review on Semi-Supervised Clustering},
  author={Cai, Jianghui and Hao, Jing and Yang, Haifeng and Zhao, Xujun and Yang, Yuqing},
  journal={Information Sciences},
  volume={632},
  pages={164--200},
  year={2023},
  publisher={Elsevier}
}

@inproceedings{wang2023goal,
  title={Goal-Driven Explainable Clustering via Language Descriptions},
  author={Wang, Zihan and Shang, Jingbo and Zhong, Ruiqi},
  booktitle={Proceedings of the 2023 Conference on Empirical Methods in Natural Language Processing (EMNLP)},
  pages={10626--10649},
  year={2023}
}

@inproceedings{van2020bipartite,
        title={Bipartite Matching in Nearly-Linear Time on Moderately Dense Graphs},
        author={van den Brand, Jan and Lee, Yin-Tat and Nanongkai, Danupon and Peng, Richard and Saranurak, Thatchaphol and Sidford, Aaron and Song, Zhao and Wang, Di},
        booktitle = {Proc. 61st Annu. IEEE Annual Symposium on Foundations of Computer Science (FOCS)},
        pages={919--930},
        year={2020},
        organization={IEEE}}

@article{feng2025bimark,
  title={BiMark: Unbiased Multilayer Watermarking for Large Language Models},
  author={Feng, Xiaoyan and Zhang, He and Zhang, Yanjun and Zhang, Leo Yu and Pan, Shirui},
  journal={ArXiv Preprint ArXiv:2506.21602},
  year={2025}
}

@article{zhang2025systematic,
  title={A Systematic Survey of Text Summarization: From Statistical Methods to Large Language Models},
  author={Zhang, Haopeng and Yu, Philip S and Zhang, Jiawei},
  journal={ACM Computing Surveys},
  volume={57},
  number={11},
  pages={1--41},
  year={2025},
  publisher={ACM New York, NY}
}

@inproceedings{zhang2024unraveling,
  title={Unraveling Privacy Risks of Individual Fairness in Graph Neural Networks},
  author={Zhang, He and Yuan, Xingliang and Pan, Shirui},
  booktitle={2024 IEEE 40th International Conference on Data Engineering (ICDE)},
  pages={1712--1725},
  year={2024},
  organization={IEEE}
}

\end{document}